\documentclass[10pt,twocolumn]{article}

\usepackage{iccv}
\usepackage{times}
\usepackage{epsfig}
\usepackage{graphicx}
\usepackage{amsmath}
\usepackage{amsthm}
\usepackage{amssymb}
\usepackage{booktabs}
\usepackage{algorithm2e}
\usepackage[pagebackref=true,breaklinks=true,colorlinks,bookmarks=false]{hyperref}
\usepackage{mathtools}
\usepackage{enumerate}

\usepackage[capitalize]{cleveref}
\crefname{section}{Sec.}{Secs.}
\Crefname{section}{Section}{Sections}
\Crefname{table}{Table}{Tables}
\crefname{table}{Tab.}{Tabs.}

\newtheorem{theorem}{Theorem}

\newtheorem{lemma}[theorem]{Lemma}

\iccvfinalcopy 

\def\ie{\emph{i.e.}}
\def\eg{\emph{e.g.}}
\def\etal{\emph{et al.}}
\def\wrt{\emph{w.r.t.}}

\ificcvfinal\fi
\begin{document}

\title{Learning with Noisy Labels: \\ Interconnection of Two Expectation-Maximizations}

\author{Heewon Kim$^{1,2}$, Hyun Sung Chang$^1$, Kiho Cho$^1$, Jaeyun Lee$^1$ and Bohyung Han$^2$\\
$^1$Samsung Advanced Institute of Technology\quad
$^2$ Seoul National Univ.\\
{\tt\small \{heewonly.kim, hyun.s.chang, kiho.cho, jaeyun91.lee\}@samsung.com, bhhan@snu.ac.kr}}
\date{}
\maketitle

\begin{abstract}
Labor-intensive labeling becomes a bottleneck in developing computer vision algorithms based on deep learning.
For this reason, dealing with imperfect labels has increasingly gained attention and has become an active field of study.
We address learning with noisy labels (LNL) problem, which is formalized as a task of finding a structured manifold in the midst of noisy data.
In this framework, we provide a proper objective function and an optimization algorithm based on two expectation-maximization (EM) cycles.
The separate networks associated with the two EM cycles collaborate to optimize the objective function, where one model is for distinguishing clean labels from corrupted ones while the other is for refurbishing the corrupted labels.
This approach results in a non-collapsing \textit{LNL-flywheel} model in the end.
Experiments show that our algorithm achieves state-of-the-art performance in multiple standard benchmarks with substantial margins under various types of label noise.
\end{abstract}

\section{Introduction}
\label{sec:intro}

The recent great success of deep learning relies heavily on large-scale high-quality datasets with manual annotations.
However, obtaining datasets with exquisite labels is extremely time-consuming and expensive.
Then, one reasonable solution to construct a large-scale dataset with labels is to extract tags or keywords from automatically crawled images using search engines~\cite{xiao2015learning, li2017webvision}.
Unfortunately, the data collected by this strategy inevitably suffer from issues related to label integrity and consistency; the ratio of corrupted labels in real-world datasets is reported to be 8.0\% to 38.5\%~\cite{xiao2015learning, li2017webvision, lee2017cleannet, song2019selfie}.

The rapid increase in the number of parameters enables deep neural networks (DNNs) to perform better on challenging tasks, but, at the same time, it makes DNN models more vulnerable to noisy labels.
Since DNNs with large capacities can fit functions with arbitrary complexity~\cite{krause2016unreasonable, arpit2017closer}, they are prone to overfitting the entire training dataset including examples with corrupted labels~\cite{zhang2017understanding}.
Hence, developing robust training methods in the presence of label noise is now recognized as a core technology to solve real-world problems.

A dataset with label noise is composed of examples with both clean and corrupted labels.
The main goal of learning with noisy labels (LNL) is to separate the training data into two groups depending on the reliability of their labels.
Once we accurately identify clean examples, it is possible to adopt existing successful techniques for training deep neural networks, \eg, semi-supervised learning~\cite{berthelot2019mixmatch}, data augmentation~\cite{zhang2018mixup,nishi2021augmentation}, to leverage pseudo-labels as well as synthetic data.
In this context, clean sample selection has been actively studied for LNL research~\cite{song2022survey}.

\begin{figure}[t]
\centering
\includegraphics[width=0.9\linewidth]{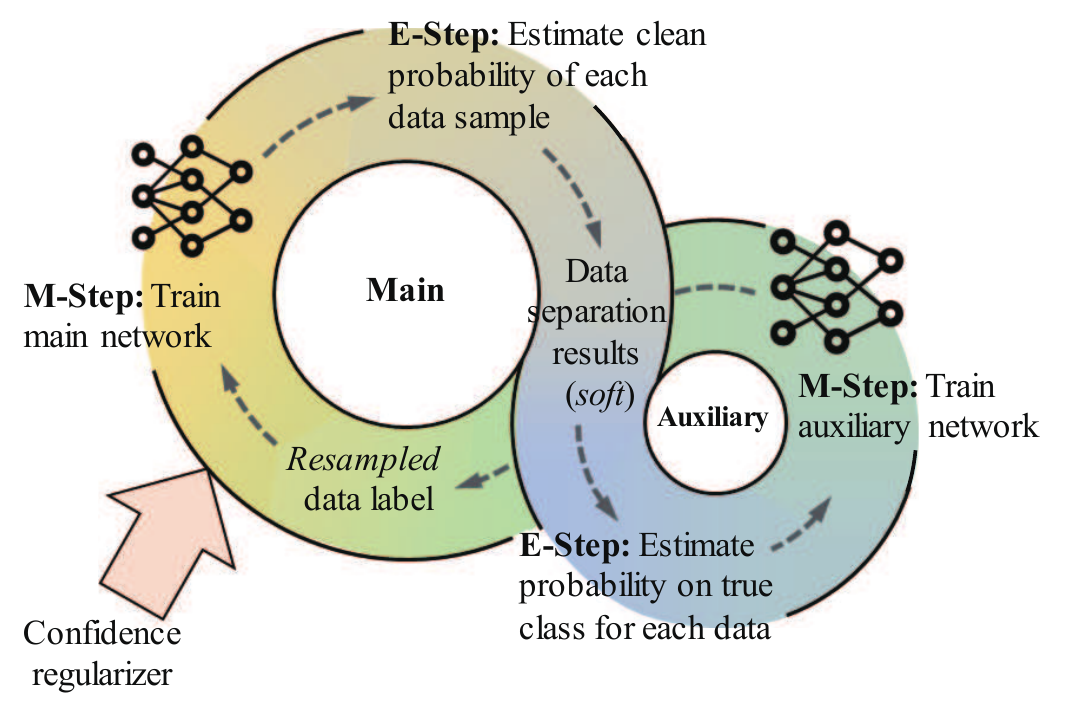}
\caption{\label{fig:teaser} An illustration of the proposed framework. Main EM and auxiliary EM are jointly connected with their distinct roles, which results in cooperative learning of LNL problem.}
\vspace{-1em}
\end{figure}

Since DNNs tend to accommodate clean samples better than corrupted ones in early stages of training but gradually overfit to noisy samples~\cite{arpit2017closer,bai2021understanding}, most sample selection methods first take small-loss examples as clean ones after a short warm-up period~\cite{han2018coteaching,jiang2018mentornet,shen2019learning,chen2019understanding,li2020dividemix}.
Then, the DNNs are trained with the selected samples, striving to widen or at least maintain the gap of the losses between the two groups.
Recent approaches employ multiple DNNs for ensemble predictions, which empirically proves to be effective in preventing incorrect selections from being accumulated~\cite{han2018coteaching,jiang2018mentornet,li2020dividemix, karim2022unicon}.
While these methods are reported to be robust to the LNL problem, there is still lack of rationale that the sample selection performance improves over the iterative processes. It remains unclear whether there exist unfortunate paths to getting stuck in a vicious cycle known as \textit{self-confirming bias}.

In this paper, we define an objective function for LNL, to fit the data with clean labels into a structured manifold and the data with corrupted labels to an outlier distribution.
We find, through non-parametric analysis, that the confidence regularizer~\cite{cheng2021cores} can properly enforce the structuredness of the manifold.
We also propose a novel algorithm to optimize the objective function through two expectation-maximization (EM) cycles, the so-called \textit{LNL-flywheel}. 
The LNL-flywheel is illustrated in Figure \ref{fig:teaser}.
Specifically, our algorithm forms the main EM cycle to maximize the overall objective function by distinguishing corrupted labels from clean ones while the auxiliary EM cycle refurbishes the training data with corrupted labels by estimating their pseudo-labels.
Note that the main and auxiliary EM cycles involve two separate networks and take turns jointly increasing the objective function until convergence.
The final outputs are given by the model associated with the auxiliary EM cycle.

Our contributions are summarized as follows:
\begin{itemize}
\item We formulate and justify an objective function of an LNL task, which consists of a likelihood term for noisy data fitting and a structural manifold regularizer. We show, through non-parametric limit analysis, that the confidence regularization is well-suited to prevent the model from collapsing into trivial solutions.
\vspace{-0.5em}
\item We provide a novel algorithm, called \textit{LNL-flywheel}, to maximize the proposed objective function. 
It runs on two EM cycles that are interconnected, in which two separate networks are trained. LNL-flywheel is a principled mechanism that the two models are jointly optimize the single objective without collapsing.
\vspace{-0.5em}
\item We achieve state-of-the-art performance in multiple standard benchmarks under environments with diverse label noise types by substantial margins. 
Unlike most of the recent methods relying on ensemble models, we only use the auxiliary network for inference, saving memory and processing time.
\end{itemize}

\section{Related work}
\label{sec:related_work}
Previous studies, which deal with LNL and are related to our flywheel model, can be categorized as loss adjustment, sample selection, and the hybrid approach for full data exploitation.

Loss adjustment methods estimate the confidence of each example for loss correction or reweighting~\cite{liu2015class,wang2018multiclass}.
The confidence level can be obtained by measuring the consistency of the model outputs.
Active bias~\cite{chang2017active} assumes that the samples with high prediction variance are likely to be corrupted ones.
As a branch of loss adjustment, label refurbishment methods \cite{reed2015bootstrap,arazo2019unsupervised} generate a refurbished label for each sample, in the form of a linear combination of the model output and the noisy label using the estimated label confidence as the weight.
Self-adaptive training~\cite{huang2020selfadaptive} applies exponential moving average to the refurbished labels for training stability.
SELFIE~\cite{song2019selfie} and AdaCorr~\cite{zheng2020error} are more selective.
They refurbish only the labels of the examples considered to be definitely corrupted.
However, model degeneracy may occur due to the bias towards easy classes or due to the erroneously refurbished labels~\cite{natarajan2013lnl,han2018coteaching}.

Sample selection approaches involve selecting correctly-labeled (clean) examples to avoid the risk of false correction.
In particular, collaborative learning and co-training processes have been widely studied.
In MentorNet~\cite{jiang2018mentornet}, a pre-trained mentor network provides the student network with a learning guide of clean samples and updates itself according to the feedback from the student.
In Co-teaching~\cite{han2018coteaching}, two differently initialized neural networks train each other by exchanging small-loss instances between themselves.
Co-teaching+~\cite{yu2019coteaching2} further employs a disagreement strategy between the two networks, and on the contrary, JoCoR~\cite{wei2020jocor} seeks to reduce the diversity of the two networks, while both are based on the learning processes of Co-teaching.
The multi-network learning methods introduced above treat small-loss instances as clean labeled data, in accordance with the observed memorization effect~\cite{arpit2017closer} of deep learning.
Cheng \etal~\cite{cheng2021cores} have proposed a confidence regularizer to make the neural network robust to noise and attempt to purify the samples by iterative sieving.

Hybrid approaches improve the data utilization by recycling the unselected data.
DivideMix~\cite{li2020dividemix}, divides the training data into a labeled set with clean samples and an unlabeled set with corrupted samples, and train the model on both the labeled and unlabeled data in a semi-supervised manner.
UNICON~\cite{karim2022unicon} have proposed uniform selection mechanism to circumvent the class imbalance in the selected clean data, and have also applied contrastive learning, to obtain robust feature representations.

These methods are related to our proposed method in the following two aspects. 
One, we distinguish clean data from corrupted ones by estimating the cleanness probability in the side of main network.
The cleanness probability is considered as a confidence level of each datapoint.
Two, the auxiliary network refurbishes noisy labels through resampling process and enables the main network to enjoy full data usage.
All these processes are formulated as an EM-based framework with a collaborative learning for LNL-flywheel.

\section{LNL-flywheel model}
\label{sec:method}
We are given a classification problem based on a training dataset of size $N$, denoted by  $\mathcal{D} \coloneqq \{(x_1, y_1), \dots , (x_N, y_N)\}$, where $x_n$ indicates a data example and $y_n \in \{1,\dots,K\}$ represents its ground-truth label.
Since training datasets contain label noise in real-world scenarios, we only observe noisy training dataset, $\widetilde{\mathcal{D}} \coloneqq \{(x_1, \tilde{y}_1), \dots , (x_N, \tilde{y}_N)\}$, which is a mixture of examples with clean and corrupted labels.
This section describes the proposed method, based on the interaction of two EM algorithms, to handle label noise in training classifiers.
We first overview our problem setup as well as an objective function in Section~\ref{sec:overview}, and provide the details of the two EM cycles in Section~\ref{sec:em-for-main} and \ref{sec:em-for-aux}. 
Figure \ref{fig:main_method} illustrates the overall training procedures of the proposed flywheel model.

\subsection{Overview}
\label{sec:overview}
From a macroscopic perspective, we consider the LNL task as an unsupervised learning conducted on the joint space ${\cal X}\times\tilde{\cal Y}$, in which clean data (\ie, data with correct labels) must be distinguished from corrupted data (\ie, data with incorrect labels) without supervision.
We accomplish this task by explaining given data with an appropriate mixture model that consists of clean data manifold as well as outlier distributions.
Let us define the mixture model in its conditional form, as
\begin{equation} \label{eq1}
p(\tilde{y}|x) = \gamma p_{\pi}(\tilde{y}|x) + (1 - \gamma) \epsilon(\tilde{y}|x),
\end{equation}
where $p_{\pi}(\cdot|x)$ is the distribution of the clean data manifold, $\epsilon(\cdot|x)$ is the distribution of corrupted data, and $\gamma$ is a proportion of clean data to be estimated.
We model $p_{\pi}(\cdot|x)$ using a classification network $g$, parameterized by $\theta$. Specifically, $g_\theta(x)$ denotes the softmax output over $K$ classes,
\ie, $g_{\theta}(x) = (g_\theta(x)[1],\dots,g_\theta(x)[K])$, given $x$ as the input. 
At the moment, for simplicity, we assume a uniform error corruption and treat $\epsilon(\tilde{y}|x)$ as a constant, \ie, $\epsilon(\tilde{y}|x)=\epsilon$, but will relax this assumption at the end of Section \ref{sec:em-for-aux} for generalization.
This formulation leads to the following equation:
\begin{align}
    p(x,\tilde{y}) & = p(x) p(\tilde{y}|x) \nonumber\\
    &= p_{data}(x) \left( \gamma g_\theta(x)[\tilde{y}] + (1-\gamma) \epsilon \right),
    \label{eq:joint-probability}
\end{align}
where we simply assume that $p(x)$ is given, \ie, $p(x)=p_\text{data}(x)$. 
We may try to fit our noisy data $\tilde{\mathcal{D}}$ to \eqref{eq:joint-probability},
however, the na\"ive log-likelihood maximization is unfortunately prone to overfitting the training data.
If the network has arbitrarily large capacity, the optimization ends up with a trivial solution ($\gamma\to 1$, $g_\theta(x)[\tilde{y}]\to p_\text{data}(\tilde{y}|x)$).
To address the challenge, we impose a structural manifold constraint that induces a sufficiently narrow subspace for data distribution to accomplish noise rejection.
Then, the objective function is rewritten as
\begin{equation}\label{eq:objective}
    \underset{\theta, \gamma}{\text{max}} ~ \mathbb{E}_{p_\text{data}(x,\tilde{y})}[\log p(x,\tilde{y})],
\end{equation}
subject to the manifold constraint. 
In practice, we adopt confidence regularizer (CR)~\cite{cheng2021cores} to realize the structural constraint of the manifold.
Section~\ref{sec:em-for-main} shows that CR prevents the model from collapsing into a trivial solution (see Lemma~\ref{thm:confidence-regularizer}).  

To solve the 
optimization problem, we propose a novel framework based on two EM cycles. 
In the main cycle, we estimate the cleanness of each example and update the mixture model parameters using the tentatively clean samples, deriving the main network $g_\theta(\cdot)$ and the mixture weight $\gamma$.
In the auxiliary cycle, we refurbish the tentatively corrupted examples by learning the auxiliary network $f_\phi(\cdot)$ via semi-supervised learning with their estimated pseudo-labels; the goal of this cycle is to increase the effective number of training data for updating the parameters of the main network.
The two EM cycles are interactive, as shown in Figure~\ref{fig:main_method}. 
They prove to be cooperative, pulling together for maximizing the single objective function in \eqref{eq:objective}. 
Once the objective has been maximized, clean data would sit on a revealed manifold, 
which creates a virtuous cycle for LNL without falling into self-confirmation bias.

\begin{figure*}[ht]
\begin{center}
\includegraphics[width=.95\textwidth]{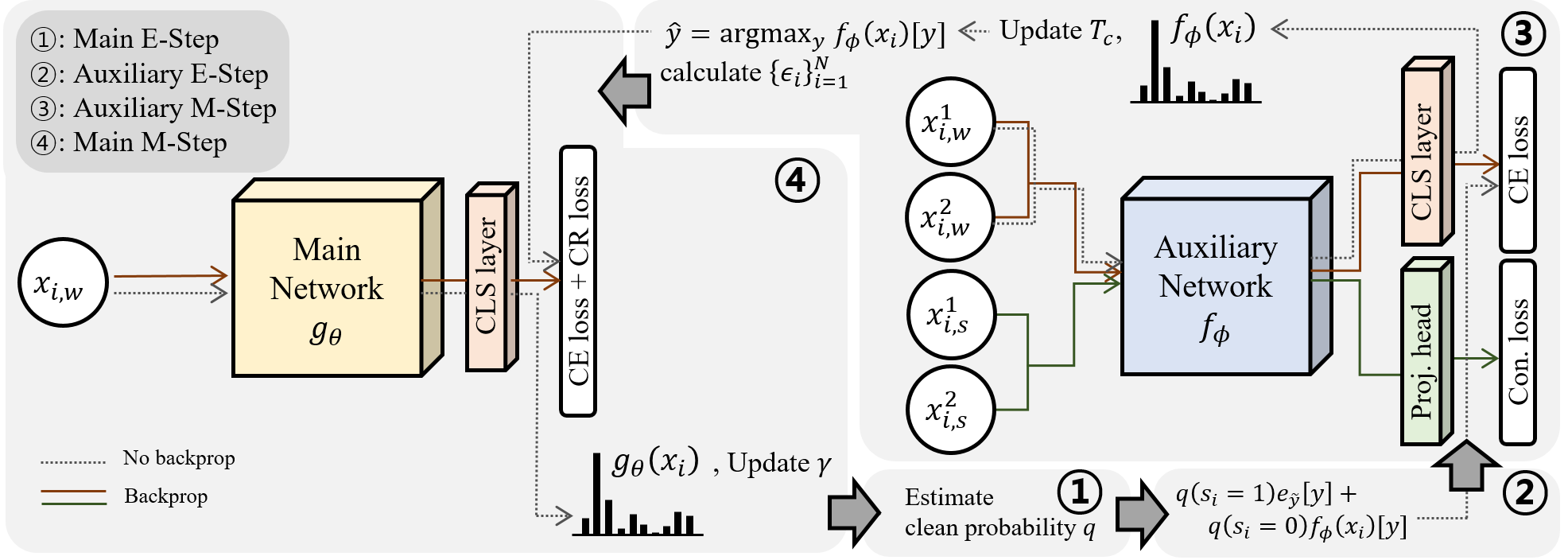}
\end{center}
\caption{Training LNL-flywheel. LNL-flywheel starts with \raisebox{.5pt}{\textcircled{\raisebox{-.9pt} {1}}} estimating clean probability of each datapoint (Eq.~\ref{eq:e-for-main}) after some initialization (Sec.~\ref{sub:implementation}).
This value is transmitted to the auxiliary side and used for \raisebox{.5pt}{\textcircled{\raisebox{-.9pt} {2}}} the true class probability estimation (Eq.~\ref{eq:qy-for-cleanish}). The true class probability serves as a supervisory signal for the auxiliary network. In \raisebox{.5pt}{\textcircled{\raisebox{-.9pt} {3}}} training the auxiliary network, contrastive learning is used together with strong augmentation (Sec.~\ref{sub:implementation}). After the parameters $f_\phi$ and $T_c$ are updated, the resampled dataset $\{(x_i,\hat{y}_i)\}_{i=1}^N$ is generated by $\hat{y}_i=\arg\max_y f_\phi(x_i)[y]$ and $\{\epsilon_i\}_{i=1}^N$ is calculated on the basis of $T_c$ (Eq.~\ref{eq:epsilon}). Then, the main network is \raisebox{.5pt}{\textcircled{\raisebox{-.9pt} {4}}} trained using the resampled dataset. In training the main network, confidence regularizer is also added as a regularizer 
(Eq.~\ref{eq:resampled-loss}). Finally, the parameters $g_\theta$ and $\gamma$ are updated, and the next cycle continues.}
\label{fig:main_method}
\end{figure*}

\subsection{EM for main network}
\label{sec:em-for-main}
The main EM cycle introduces a binary latent variable $s_i$ indicating whether the corresponding example $(x_i, \tilde{y}_i)$  is from clean data manifold $(s_i=1)$ or the corrupted data distribution $(s_i=0)$, and iterates the following two steps at every cycle $t$.

In E-Step, we estimate $s_i$ based on the current main model, $g_\theta(\cdot)$, which is given by the Bayes' rule as
\begin{align}
    \label{eq:e-for-main}
    q^{(t)}(s_i=1) & \coloneqq p^{(t-1)}(s_i=1 | x_i, \tilde{y}_i) \\
    &= \frac{\gamma^{(t-1)}g_\theta^{(t-1)}(x_i)[\tilde{y}_i]}{\gamma^{(t-1)}g_{\theta}^{(t-1)}(x_i)[\tilde{y}_i]+(1-\gamma^{(t-1)})\epsilon}. \nonumber
\end{align}

In M-Step, we update model parameters based on the membership of each example. 
For each sample, an appropriate distribution, either $\gamma g_\theta(x)[\tilde{y}]$ or $(1-\gamma)\epsilon$, is selected from the mixture model according to whether $s=1$ or $s=0$ and the log likelihood is computed based on the selected distribution.
The final evaluation takes the form of expectation because the data separation is only available in the probabilistic sense. 
The M-step is therefore to solve the following problem: 
\begin{multline}
    \underset{\theta, \gamma}{\max} \;
    \frac{1}{N}\sum_{i=1}^N \Bigl\{q^{(t)}{ (s_i=1)}\log\left(\gamma g_\theta(x_i)[\tilde{y}_i]\right)\\
    + q^{(t)}{ (s_i=0)}\log\left((1-\gamma)\epsilon\right)\Bigr\}, 
    \label{eq:m-for-main}
\end{multline}
subject to the manifold constraint.
Because $\gamma$ and $\theta$ are disjoint in \eqref{eq:m-for-main}, the maximization can be conducted separately. 
We obtain the closed-form solution for $\gamma$, which is given by
\begin{align}
    \gamma^{(t)} = \frac{1}{N}\sum_{i=1}^N q^{(t)}{ (s_i=1)},
    \label{eq:gamma}
\end{align}
and the optimization of $\theta$ is achieved by training $g_\theta(\cdot)$ with the following loss function: 
\begin{align}
    L(\theta) =  \; -\frac{1}{N}\sum_{i=1}^N q^{(t)}{ (s_i=1)}\log g_\theta(x_i)[\tilde{y}_i],
    \label{eq:reweighted-loss}
\end{align}
subject to the manifold constraint. 
Note that \eqref{eq:reweighted-loss} 
corresponds to the cross-entropy evaluated with reweighted samples, which makes the network training depend more on clean samples than corrupted ones. This is reasonable but the model generalization issue arises if the number of clean samples is too small (\eg, with high noise rate).
Hence, we replace reweighting with resampling~\cite{an2021resampling} to utilize the whole training dataset, by generating new samples using the auxiliary network. 

Let us denote, by $(x_i, \hat{y}_i^{(t)})$, the data samples generated from the auxiliary network at the cycle $t$. 
Then, we modify the loss function in \eqref{eq:reweighted-loss} to
\begin{align}
    L_{\text{main}, g}^{(t)} =  -\frac{1}{N}\sum_{i=1}^{N} \log g_\theta(x_i)[\hat{y}_i^{(t)}] +\lambda\eta{(g_\theta)}
    \label{eq:resampled-loss}
\end{align}
so that it is based on the resampled dataset. In \eqref{eq:resampled-loss}, we employ CR~\cite{cheng2021cores}, denoted by $\eta(g_\theta)$, as the regularizer relaxed from the manifold constraint via Lagrange multiplier $\lambda$.
The CR is defined as
\begin{align}
\eta(g_\theta)=\frac{1}{N}\sum_{i=1}^N {\mathbb E}_{p_{data}(\hat{y})}[\log g_\theta(x_i)[\hat{y}]].
\end{align}
We argue that it has a property that it reveals the clean data manifold in the mixed dataset (see Lemma~\ref{thm:confidence-regularizer}). 

\begin{lemma}
\label{thm:confidence-regularizer}
Suppose a neural classifier $g_\theta(x)$ which has enough capacity to realize an arbitrary function.
We consider the following regularized loss function for training:
\begin{multline*}
    L = -{\mathbb E}_{p_{\rm data}(x,\hat{y})}\Bigl[\log\, g_\theta(x)[\hat{y}]\Bigr] \\
    + \lambda {\mathbb E}_{p_{\rm data}(x)} {\mathbb E}_{p_{\rm data}(\hat{y})} \Bigl[\log\, g_\theta(x)[\hat{y}]\Bigr].
\end{multline*}
If the data mixture distribution has the same form as our model, \ie,
$p_{\rm data}(\hat{y}|x) = \gamma' p_\pi(\hat{y}|x)+(1-\gamma')\epsilon$
for some $\gamma'\in (0,1)$, and if the labels in the data are uniformly distributed, \ie, $p_{\rm data}(\hat{y}) = \epsilon$,
there exists $\lambda$ that makes the softmax output $g_\theta(x)[\hat{y}]$ approach $p_\pi(\hat{y}|x)$ in the non-parametric limit.
\end{lemma}
\begin{proof}
See Appendix.
\end{proof}

\subsection{EM for auxiliary network}
\label{sec:em-for-aux}

In the LNL-flywheel, the auxiliary network conducts resampling process to provide a sufficiently large amount of training data to the main network.
The auxiliary network also runs on an EM cycle.

In E-Step, it attempts to estimate the probability on the true class of the given data $(x_i, \tilde{y}_i)$, as follows:
\begin{align}
    q_i^{(t)}(y)&=p^{(t-1)}(y|x_i, \tilde{y}_i)\nonumber \\
    &=\sum_{s_i=0}^{1} p^{(t-1)}(s_i|x_i,\tilde{y}_i)p^{(t-1)}(y|x_i,\tilde{y}_i,s_i).
    \label{eq:qy-for-cleanish-interim}
\end{align}
If $s_i=1$, $\tilde{y}_i$ should be identical to the true class, and therefore $p(\cdot|x_i,\tilde{y}_i,s_i=1)=e_{\tilde{y}_i}$, where $e_c$ denotes a one-hot vector with 1 at the $c$th entry.
If $s_i=0$, we throw away the corrupted label and estimate the true class based on the input $x_i$, so we model $p(\cdot|x_i, \tilde{y}_i, s_i=0)=p(\cdot|x_i)=f_\phi(x_i)$, where $f_\phi(x_i)$ denotes the softmax output of the auxiliary network.
In the side of the main network, we already computed the cleanness probability $q^{(t)}(s_i)=p^{(t-1)}(s_i|x_i, \tilde{y}_i)$ for each example in our noisy dataset (see Eq.~\ref{eq:e-for-main}).
We reuse it here, simplifying
\eqref{eq:qy-for-cleanish-interim} to
\begin{align}
    q_i^{(t)}(y) = &~ q^{(t)}(s_i=1)e_{\tilde{y}_i}[y]  \label{eq:qy-for-cleanish} \\
    & + q^{(t)}(s_i=0)f^{(t-1)}_\phi(x_i)[y]. \nonumber
\end{align}

In M-Step, the estimated probability $q_i^{(t)}(y)$ weighs each example in finding the parameters relevant to class $y$. 
Then, we maximize the likelihood, ${\cal L}_{\rm aux}^{(t)}$, which is given by
\begin{align}
    {\cal L}_{\rm aux}^{(t)}=\frac{1}{N}\sum_{i=1}^{N}\sum_{y=1}^{K}q_i^{(t)}(y)\log p(y,\tilde{y}_i|x_i).
\end{align}
By applying chain rule to $p(y,\tilde{y}_i|x_i)$, we obtain
\resizebox{1.0\linewidth}{!}{
  \begin{minipage}{\linewidth}
  \begin{align*}
    {\cal L}_{\rm aux}^{(t)} &=\frac{1}{N}\sum_{i=1}^N\sum_{y=1}^K q_i^{(t)}(y)\Bigl(\log p(y|x_i)+\log p(\tilde{y}_i|x_i,y)\Bigr)  \\
    &= \frac{1}{N} \sum_{i=1}^N \sum_{y=1}^K q_i^{(t)}(y)\Bigl(\log f_\phi(x_i)[y]+\log T_{x_i}[y,\tilde{y}_i]\Bigr), 
\end{align*}
  \end{minipage}
}
where $T_{x_i}[y,\tilde{y}_i] \coloneqq p(\tilde{y}_i|x_i,y)$ denotes the label corruption probability. We assume that $T_{x_i}$ remains constant for all $x_i$, and implement it by a $K \times K$ matrix.
Since the two models, $f$ and $T$, do not share any parameters between themselves, ${\cal L}_{\rm aux}^{(t)}$ can be separately maximized $\wrt$ each model. 

Maximizing ${\cal L}_{\rm aux}^{(t)}$ $\wrt$ $f$ is equivalent to training the auxiliary network with the loss
\begin{align}
    L_{{\rm aux}, f}^{(t)} &=-\frac{1}{N}\sum_{i=1}^N\sum_{y=1}^K q_i^{(t)}(y)\log f_\phi(x_i)[y],
    \label{eq:ce-loss-for-semisup}
\end{align}
which is a cross-entropy with $q_i^{(t)}(y)$ as the training target.
Note that the training target computed as in \eqref{eq:qy-for-cleanish} has a nice connection to {\em label guessing}~\cite{berthelot2019mixmatch} and {\em label smoothing} \cite{li2020dividemix}, two practices exercised in modern semi-supervised learning.
In training $f_\phi$, we also adopt MixUp~\cite{zhang2018mixup}, another common practice for semi-supervised learning, to carry out smooth generalization from discrete clean  examples to the continuous space.

\begin{figure}[t]
\begin{center}
\small
\begin{tabular}{ccc}
\multicolumn{3}{c}{\hspace*{-2.5em} Scale: \includegraphics[width=.25\textwidth]{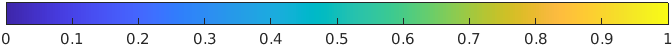} } \\
\\
\includegraphics[width=.29\linewidth]{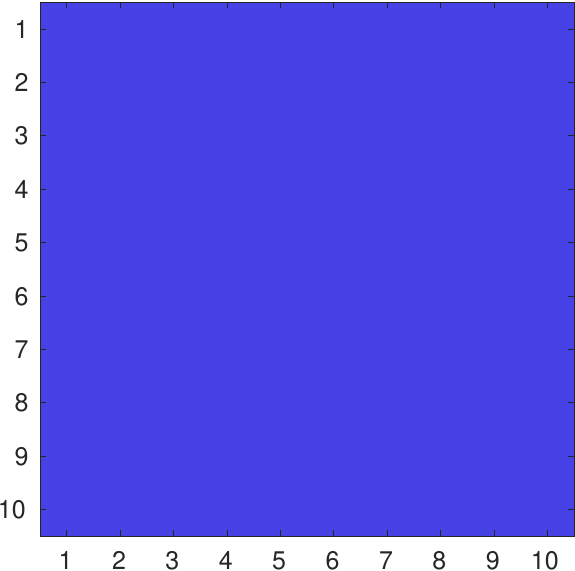} &
\includegraphics[width=.29\linewidth]{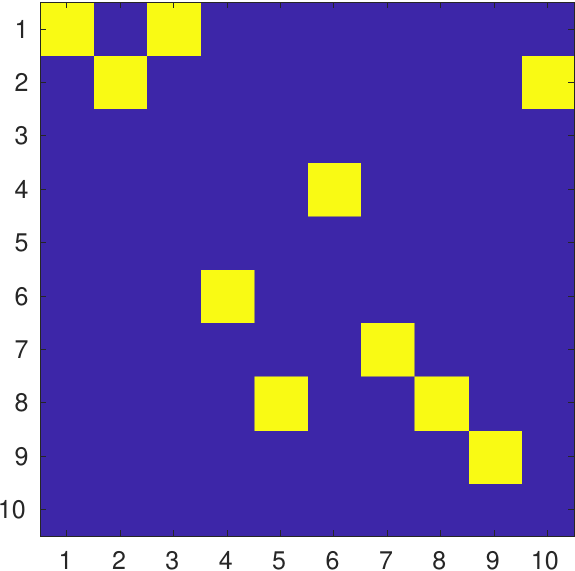} &
\includegraphics[width=.29\linewidth]{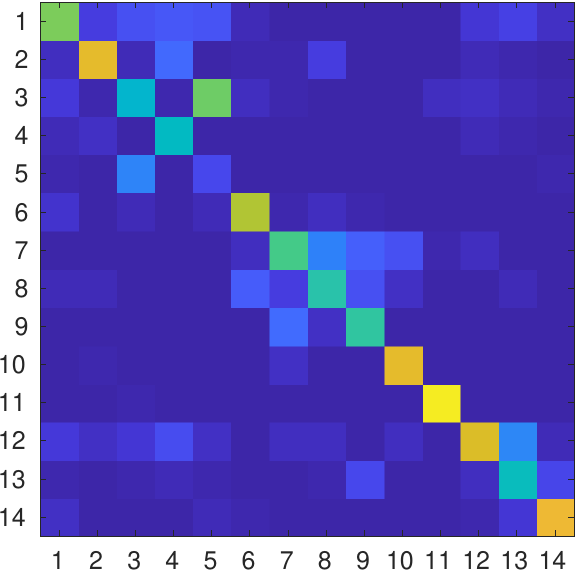} \\
\includegraphics[width=.29\linewidth]{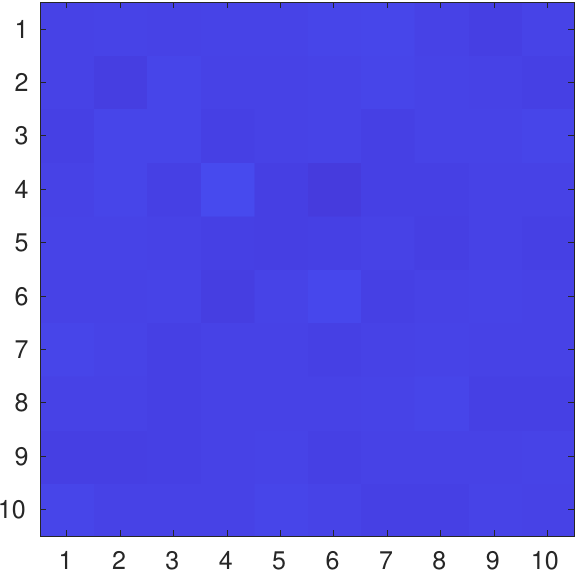} &
\includegraphics[width=.29\linewidth]{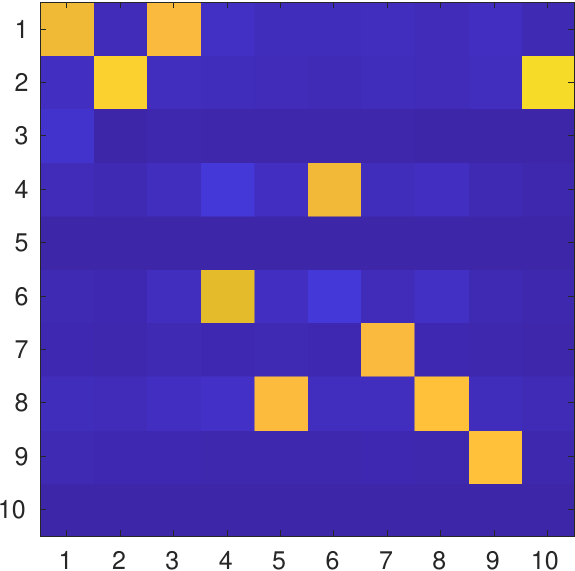} &
\includegraphics[width=.29\linewidth]{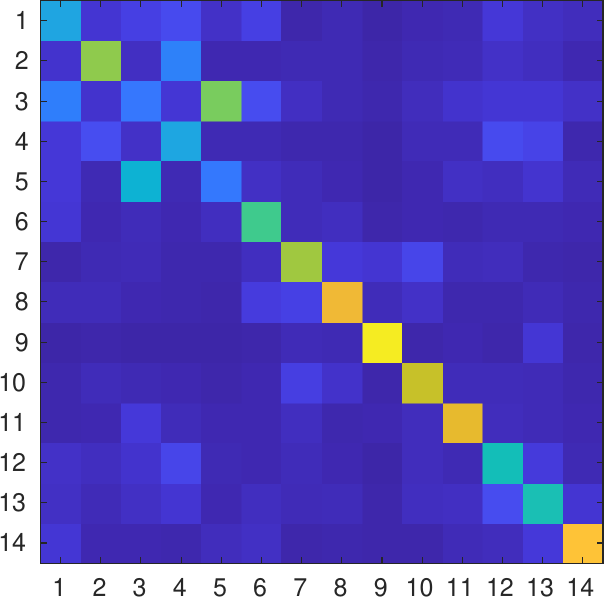} \\
CIFAR-10, Sym & CIFAR-10, Asym & Clothing1M
\end{tabular}
\end{center}
\caption{Visualization of label corruption matrices for synthetic (CIFAR-10) and real (Clothing1M) noises. \textbf{Upper.} Ground-truth, \textbf{Lower.} Estimated.  Color scale is displayed at the top. In each figure, the $i$th column corresponds to the corruption probability of the $i$th source class, \ie, $T_{c}(i,\cdot)$. The ground-truth of Clothing1M is obtained by manual inspection over a small fraction of the entire dataset \cite{xiao2015learning}. } 
\label{fig:transition}
\end{figure}
By solving the maximization problem $\wrt$ $T$ with the constraint that $T$ must be a valid  transition probability matrix (\ie, $\sum_{y'}T[y,y']=1$, $\forall y$ and $T[y,y'] \geq 0$, $\forall y$, $y'$), we obtain $T^{(t)}[y,y']=\sum_{i: \tilde{y}_i=y'}q_i^{(t)}(y) / \sum_i q_i^{(t)}(y).$

With the availability of $T$, we can refine $\epsilon$, assumed to be constant until now, for the next iteration of the main EM.
To lay a connection between $T$ and $\epsilon$, we first note that, in \eqref{eq:qy-for-cleanish}, $q_i^{(t)}(y)$ has two parts, \ie, \\
\resizebox{1.0\linewidth}{!}{
  \begin{minipage}{\linewidth}
  \begin{align*}
    q_i^{(t)}(y) = \underbrace{q^{(t)}(s_i=1)e_{\tilde{y}_i}[y]}_{\text{clean part}} + \underbrace{q^{(t)}(s_i=0)f^{(t-1)}_\phi(x_i)[y],}_{\text{corrupted part}}
  \end{align*}
  \end{minipage}
}
and modify $T$ and denote the modified expression by $T_c$ as follows, so that it is only based on corrupted examples:
\begin{align*}
T_c^{(t)}[y,y']=\frac{\sum_{i:\tilde{y}_i=y'}q^{(t)}(s_i=0)f^{(t-1)}_\phi(x_i)[y]}{\sum_{i}q^{(t)}(s_i=0)f^{(t-1)}_\phi(x_i)[y]}.
\end{align*}
Then, by definition, $\epsilon_i$ is calculated as
\begin{align}
    \epsilon_i^{(t)} &=p^{(t)}(\tilde{y}_i|x_i,s_i=0)\nonumber\\
    &= \sum_{y=1}^K p^{(t)}(y|x_i,s_i=0)p^{(t)}(\tilde{y}_i|x_i,y,s_i=0)\nonumber\\
    &=\sum_{y=1}^K f^{(t)}_\phi(x_i)[y]T_c^{(t)}[y,\tilde{y}_i].
    \label{eq:epsilon}
\end{align}
The main network can utilize $\epsilon_i^{(t)}$ in place of $\epsilon$ in \eqref{eq:e-for-main}, which is appropriate for the real world label noise that is not uniformly distributed.
Figure \ref{fig:transition} shows several examples of the label corruption matrices estimated in our experiments.
The classes represented by the numbers in figure \ref{fig:transition} are as follows: 1--Airplane, 2--Automobile, 3--Bird, 4--Cat, 5--Deer, 6--Dog, 7--Frog, 8--Horse, 9--Ship, 10--Truck for CIFAR-10; 1--T-Shirt, 2--Shirt, 3--Knitwear, 4--Chiffon, 5--Sweater, 6--Hoodie, 7--Windbreaker, 8--Jacket, 9--Down Coat, 10--Suit, 11--Shawl, 12--Dress, 13--Vest, 14--Underwear for Clothing1M.
The estimation turns to work well not only for symmetric but also for asymmetric as well as real world label noise.

\section{Experiments}
\label{sec:experiment}

\subsection{Implementation details}
\label{sub:implementation}
This section discusses the specifics of the implementation of our algorithm.

\paragraph{Initialization}
In advance of the start of the first cycle, some model parameters need to be initialized.
We initialize the networks, \ie, $g_\theta$ and $f_\phi$, by the warm-up training with a few epochs.
For $\gamma$, we use the training accuracy of $f_\phi$ evaluated after warm-up period.
Considering the early-stage learning tendency, this roughly approximates the ratio of clean labels.
For $\epsilon$, we initially set it to $1/K$, where $K$ is the number of classes, and iteratively update it using \eqref{eq:epsilon}. 

\begin{table*}[t]
  \centering
  \caption{Test accuracies (\%) on CIFAR-10/100 with symmetric and asymmetric label noise. Most results are extracted from \cite{karim2022unicon} while results with $^{\dagger}$ are reported in their respective papers.}
  \small
  \begin{tabular}{l|cccc|ccc|cccc|ccc}
    \toprule
    Dataset & \multicolumn{7}{c|}{CIFAR-10} & \multicolumn{7}{c}{CIFAR-100}\\
    \midrule
    Noise type & \multicolumn{4}{c|}{Symmetric} & \multicolumn{3}{c|}{Asymmetric} & \multicolumn{4}{c|}{Symmetric} & \multicolumn{3}{c}{Asymmetric}\\
    Noise level & 20\% & 50\% & 80\% & 90\% & 10\% & 30\% & 40\% &  20\% & 50\% & 80\% & 90\% & 10\% & 30\% & 40\% \\
    \midrule
    LDMI~\cite{xu2019ldmi} & 88.3 & 81.2 & 43.7 & 36.9& 91.1 & 91.2 & 84.0 & 58.8 & 51.8 & 27.9 & 13.7 & 68.1 & 54.1 & 46.2\\
    MixUp~\cite{zhang2018mixup} & 95.6 & 87.1 & 71.6 & 52.2 & 93.3 & 83.3 & 77.7 & 67.8 & 57.3 & 30.8 & 14.6 & 72.4 & 57.6 & 48.1\\
    JPL~\cite{kim2021jpl} & 93.5 & 90.2 & 35.7 & 23.4 & 94.2 & 92.5 & 90.7 & 70.9 & 67.7 & 17.8 & 12.8  & 72.0 & 68.1 & 59.5\\
    PENCIL~\cite{kun2019pencil} & 92.4 & 89.1 & 77.5 & 58.9 & 93.1 & 92.9 & 91.6 & 69.4 & 57.5 & 31.1 & 15.3 & 76.0 & 59.3 & 48.3 \\
    MOIT~\cite{diego2021moit} & 94.1 & 91.1 & 75.8 & 70.1 & 94.2 & 94.1 & 93.2  & 75.9 & 70.1 & 51.4 & 24.5 & 77.4 & 75.1 & 74.0 \\
    DivideMix~\cite{li2020dividemix} & 96.1 & 94.6 & 93.2 & 76.0 & 93.8 & 92.5 & 91.7 & 77.3 & 74.6 & 60.2 & 31.5  & 71.6 & 69.5 & 55.1 \\
    Sel-CL+$^\dagger$~\cite{li2022selcl} & 95.5 & 93.9 & 89.2 & 81.9 & 95.6 & 94.5 & 93.4 & 76.5 & 72.4 & 59.6 & 48.8  & 78.7 & 76.4 & 74.2 \\
    ELR~\cite{liu2020elr} & 95.8 & 94.8 & 93.3 & 78.7 & 95.4 & 94.7 & 93.0 & 77.6 & 73.5 & 60.8 & 33.4 & 77.3 & 74.6 & 73.2\\
    SFT$^\dagger$~\cite{wei2022sft} & - & 94.9 & - & - & - & - & 93.7 & - & 75.2 & - & -  & - & - & 74.9 \\
    UNICON~\cite{karim2022unicon} & 96.0 & 95.6 & 93.9 & 90.8 & 95.3 & 94.8 & 94.1 & 78.9 & 77.6 & 63.9 & 44.8 & 78.2 & 75.6 & 74.8 \\
    \midrule
    \textbf{OURS} & \textbf{97.0} & \textbf{96.7} & \textbf{95.9} & \textbf{94.4} & \textbf{97.3} & \textbf{96.6} & \textbf{96.2} & \textbf{79.7} & \textbf{77.7} & \textbf{71.9} & \textbf{58.8} & \textbf{80.9} & \textbf{79.1} & \textbf{78.6} \\
    \bottomrule
  \end{tabular}
  \label{tab:cifar-sym-asym}
\end{table*}

\paragraph{Data augmentation}
As shown in Figure \ref{fig:main_method}, we employ data augmentation for the auxiliary network, to push the envelope of performance systematically.
We consider two types of augmentations \cite{nishi2021augmentation}, \ie, weak augmentation (standard random crops and flips) and strong augmentation (14 random transformations; see \cite{cubuk2020randaugment} for details), and utilize them in two ways. 
\begin{enumerate}[(i)]
\item E-Step: To evaluate the network prediction $f_\phi^{(t-1)}(x_i)$ in \eqref{eq:qy-for-cleanish}, we average the predictions across two weakly augmented copies of $x_i$. 
\item M-Step: We train $f_\phi$ with the MixUp \cite{zhang2018mixup} interpolation on weakly augmented copy and with contrastive loss~\cite{chen2020contrastive, khosla2020contrastive} on strongly augmented copy.
Therefore, the total loss function for training the auxiliary network is $L_{{\rm total},f}=L_{{\rm aux},f}+\alpha L_{\rm con}$, where $L_{\rm con}$ is the contrastive loss and $\alpha$ is the corresponding coefficient.
\end{enumerate}

\paragraph{Hyperparameters}
The confidence regularizer coefficient $\lambda$ and the contrastive loss coefficient $\alpha$ remain as hyper-parameters. 
They are set to 3 and 0.025, respectively.

\paragraph{Inference}
We use the auxiliary network for inference.
The auxiliary network is responsible for generating data for the main network to train with.
This implies that it behaves like a teacher and thus is naturally expected to be better than the main network for inference, despite the naming.
Using a single network for inference creates advantages in terms of resources of the inference engine, in comparison with recent state-of-the-art methods~\cite{li2020dividemix, karim2022unicon} that take two-model strategy and use model ensemble for inference.

\subsection{Datasets and training details}
\label{sub:datasets}

We evaluate our method on five standard benchmarks, CIFAR-10, CIFAR-100~\cite{krizhevsky2009learning}, Tiny-ImageNet~\cite{ya2015tiny}, Clothing1M~\cite{xiao2015learning}, and WebVision~\cite{li2017webvision}.

\paragraph{CIFAR-10 and CIFAR-100} These two datasets consist of 50k training images and 10k test images with the size of $32\times 32$. The numbers in the suffix of dataset names indicate the number of classes. 
We artificially inject three types of label noise, \ie, symmetric, asymmetric, and instance-dependent label noise.
Symmetric noise is generated by flipping labels, with the probability specified by the noise rate, evenly to all possible classes.
Asymmetric noise is designed to simulate real-world label noise, only allowing labels to be replaced by similar classes with a given probability.
Instance-dependent label noise is a more realistic noise model in which each instance has a different flip rate. To generate the instance-dependent type of noise, we followed the method proposed in~\cite{xia2020partdependent, cheng2021cores}.
We evaluate algorithms for several different noise rates.
We employ an 18-layer PreAct ResNet \cite{he2016resnet} and train it using SGD with a momentum of 0.9, a weight decay of 0.0005, and a batch size of 128.
The network is trained for 300 epochs.
We set the learning rate initially to 0.02 and decay it according to the cosine annealing schedule without restart \cite{ilya2017sgdr}. The warm-up period is 10 both for CIFAR-10 and CIFAR-100.

\paragraph{Tiny-ImageNet}
There are 200 classes containing 500 images per class, and the image size is $64\times 64$. 
We use SGD with a momentum of 0.9, a weight decay of 0.0005, and a batch size of 32. 
We choose the initial learning rate of 0.01 and set its decay factor to 0.5 per 25 epochs after 100 epochs. 
We train our model for 200 epochs with a warm-up period of 10 epochs.

\paragraph{Clothing1M} This is a real world noisy dataset collected from online shopping websites.
This dataset contains 1M clothes images in 14 classes with unknown noise rate.
The ResNet-50 pre-trained with ImageNet \cite{russakovsky2015imagenet} was employed for fair comparison with the existing works.
We use SGD with a momentum of 0.9, a weight decay of 0.001, and a batch size of 32. The learning rate starts with 0.001 and reduces by a factor of 10 after 40 epochs.
The network is trained for 80 epochs.
In each epoch, we only use 1,000 mini-batches sampled from the training data while ensuring that the labels are balanced, following DivideMix \cite{li2020dividemix}.

\paragraph{WebVision} We test another real-world noisy dataset that contains 2.4 million images crawled from the web using the 1,000 concepts in ImageNet ILSVRC12 \cite{russakovsky2015imagenet}. Following previous work \cite{chen2019understanding}, we assess our methods on the first 50 classes of the Google image subset, using the Inception-ResNet-v2 \cite{szegedy2017inception-resnet}.
We use SGD with a momentum of 0.9, a weight decay of 0.0005, and a batch size of 32. The initial learning rate is 0.02 and its decaying rate is 0.5/25 epochs after 50 epochs. The network is trained for 100 epochs.

\begin{figure*}[t]
\begin{center}
\small
\begin{tabular}{ccc}
\includegraphics[width=.25\linewidth]{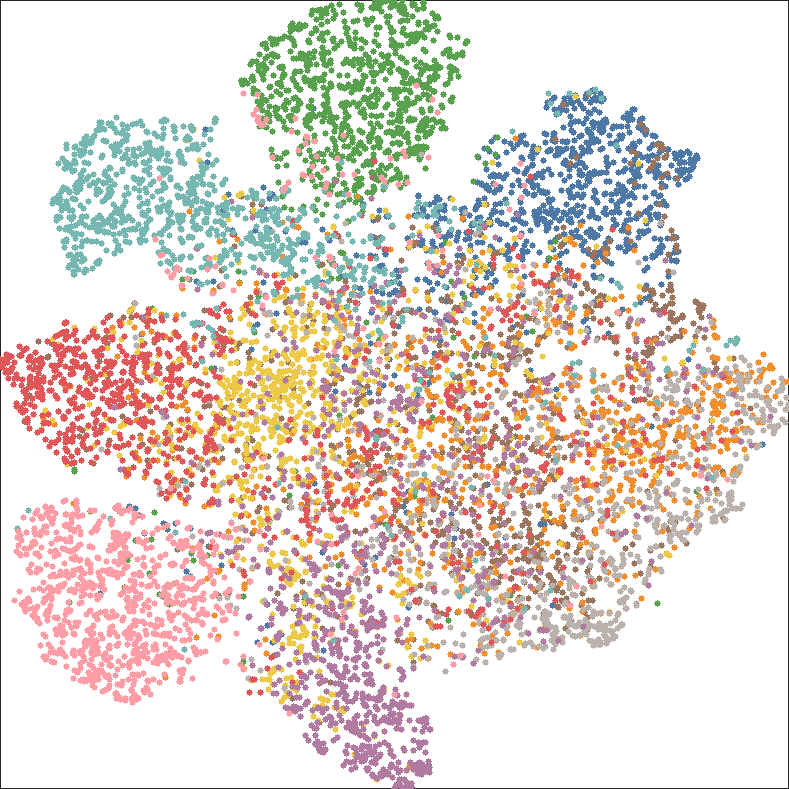} \hspace{5mm} &
\includegraphics[width=.25\linewidth]{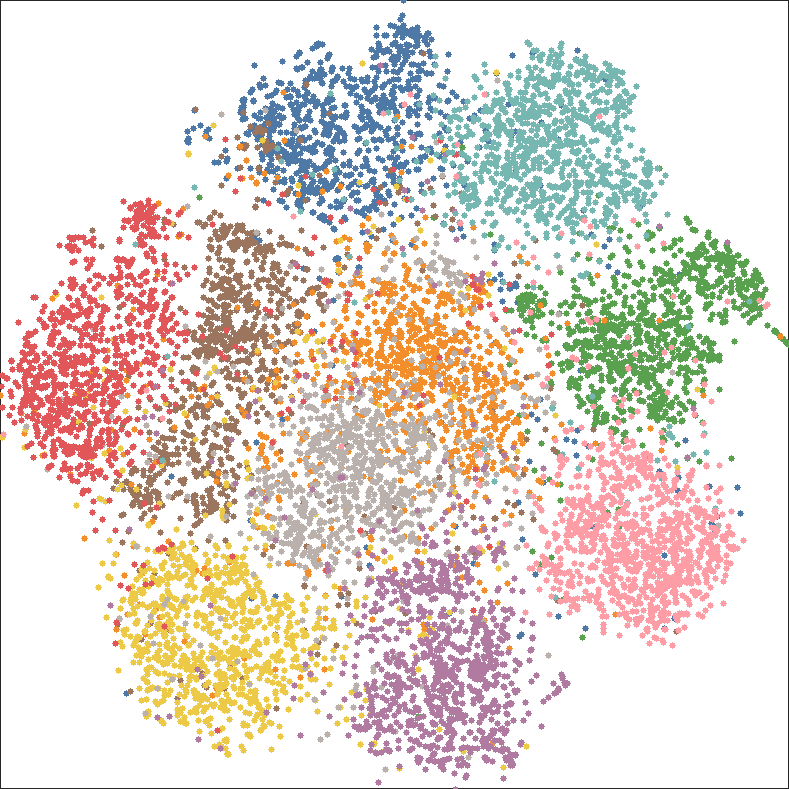} \hspace{5mm} &
\includegraphics[width=.25\linewidth]{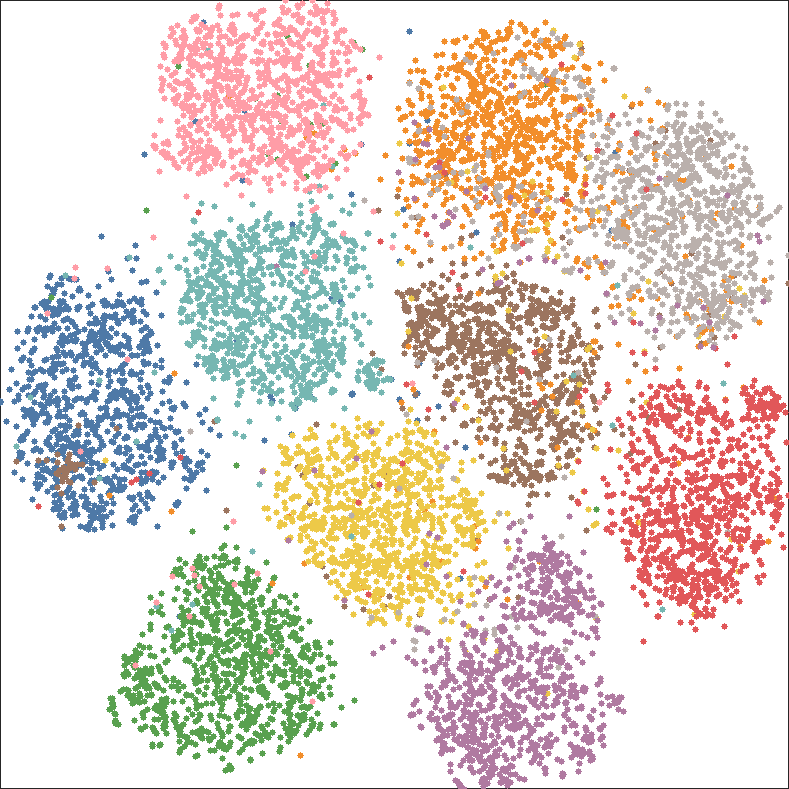}\\
(a) DivideMix~\cite{li2020dividemix} & (b) UNICON~\cite{karim2022unicon} & (c) LNL-flywheel
\end{tabular}
\end{center}
\caption{t-SNE visualizations of network features of test images. Networks are trained  for 300 epochs using (a) DivideMix (b) UNICON and (c) LNL-flywheel, on CIFAR-10 dataset with $90\%$ symmetric noise.}
\label{fig:tsne}
\end{figure*}

\begin{table}[t]
  \centering
  \caption{Test accuracies (\%) on CIFAR-10/100 with instance-dependent label noise.
  Most results are extracted from ~\cite{zhu2021cal}, while results with $^\dagger$ are reported in their respective papers. Methods with $^*$ are targeted for instance-dependent label noise.
  }
  \small
  \begin{tabular}{l|ccc|ccc}
    \toprule
    Dataset & \multicolumn{3}{c|}{CIFAR-10} & \multicolumn{3}{c}{CIFAR-100}\\
    Noise level & 20\% & 40\% & 60\% & 20\% & 40\% & 60\% \\
    \midrule
    Co-teaching+~\cite{yu2019coteaching2} & 89.8 & 73.8 & 59.2 & 41.7 & 24.5 & 12.6\\
    JoCoR~\cite{wei2020jocor} & 88.8 & 71.6 & 63.5 & 43.7 & 24.0 & 13.1\\
    Reweight-R~\cite{xia2019anchor} & 90.0 & 84.1 & 72.2 & 58.0 & 43.8 & 36.1\\
    Peer Loss~\cite{liu2020peer} & 89.1 & 83.3 & 74.5 & 61.2 & 47.2 & 31.7\\
    CORES$^{2*}$~\cite{cheng2021cores} & 91.1 & 83.7 & 77.7 & 66.5 & 59.0 & 38.6\\
    CAL$^*$\cite{zhu2021cal}& 92.0 & 85.0 & 79.8 & 69.1 & 63.2 & 43.6\\
    kMEIDTM$^{*\dagger}$~\cite{cheng2022instance} & 92.3 & 85.9 & - & 69.2 & 66.8 & -\\
    SFT$^\dagger$~\cite{wei2022sft}&91.4&90.0&-&71.8&69.9&-\\
    InstanceGM$^{*\dagger}$~\cite{garg2023instancegm} & 96.7 & 96.4 & - & 79.7 & \textbf{78.5} & - \\
    \midrule
    \textbf{OURS} & \textbf{96.7} & \textbf{96.8} & \textbf{96.1} & \textbf{79.8} & 77.9 & \textbf{76.2} \\
    \bottomrule
  \end{tabular}

    \vspace{-1mm}
  \label{tab:cifar-instance}
\end{table}

\begin{table}[t]
  \centering
  \caption{Test accuracies (\%) on Tiny-ImageNet under symmetric noise settings. Comparison with results reported by \cite{karim2022unicon} with the highest (Best) and the average (Avg.) test accuracy over the last 10 epochs.}
  \small
  \begin{tabular}{l|cc|cc|cc}
    \toprule
    Noise level & \multicolumn{2}{c|}{0\%} & \multicolumn{2}{c|}{20\%} & \multicolumn{2}{c}{50\%}\\
    Method & Best & Avg. & Best & Avg. & Best & Avg.\\
    \midrule
    F-correction~\cite{patrini2017making} & - & - & 44.5 & 44.4 & 33.1 & 32.8\\
    MentorNet~\cite{jiang2018mentornet} & - & - & 45.7 & 45.5 & 35.8 & 35.5\\
    Co-teaching+~\cite{yu2019coteaching2} & 52.4 & 52.1 & 48.2 & 47.7 & 41.8 & 41.2 \\
    
    M-correction~\cite{arazo2019unsupervised} & 57.7 & 57.2 & 57.2 & 56.6 & 51.6 & 51.3 \\
    NCT~\cite{sarfraz2021nct} & 62.4 & 61.5 & 58.0 & 57.2 & 47.8 & 47.4 \\
    UNICON~\cite{karim2022unicon} & \textbf{63.1} & \textbf{62.7} & 59.2 & 58.4 & 52.7 & 52.4 \\
    \midrule
    \textbf{OURS} & \textbf{63.1} & \textbf{62.7} & \textbf{60.3} & \textbf{60.0} & \textbf{53.4} & \textbf{53.0} \\
    \bottomrule
  \end{tabular}
  \label{tab:tiny}
\end{table}

\begin{table}[t]
  \centering
  \caption{Top-1 test accuracies (\%) on Clothing1M. Most results are extracted from ~\cite{karim2022unicon}, while results with $^\dagger$ are reported in their respective papers.
  }
  \small
  \begin{tabular}{lc}
  \toprule
  Method & Accuracy\\
  \midrule
  PENCIL~\cite{kun2019pencil} & 73.49 \\
  JPL~\cite{kim2021jpl} & 74.15\\
  InstanceGM$^\dagger$~\cite{garg2023instancegm} & 74.40\\
  
  DivideMix~\cite{li2020dividemix} & 74.76 \\ ELR~\cite{liu2020elr} & 74.81 \\
  UNICON~\cite{karim2022unicon} & 74.98 \\
  SFT$^\dagger$~\cite{wei2022sft} & \textbf{75.08}\\
  \midrule
  \textbf{OURS} & 75.04\\
  \bottomrule
  \end{tabular}
  \label{tab:clothing}
\end{table}

\begin{table}[t]
  \centering
  \caption{Top-1 and Top-5 test accuracies (\%) on WebVision. Most results are extracted from ~\cite{karim2022unicon}, while results with $^\dagger$ are reported in their respective papers.}
  \small
  \begin{tabular}{l|cc}
  \toprule
  Method & \textit{Top-1} & \textit{Top-5} \\
  \midrule
  Iterative-CV~\cite{wang2018iterative} & 65.24 & 85.34\\
  DivideMix~\cite{li2020dividemix}& 77.32 & 91.64\\
  ELR~\cite{liu2020elr}&77.78 & 91.68\\
  MOIT~\cite{diego2021moit} & 78.76 & -\\
  UNICON~\cite{karim2022unicon} & 77.60 & 93.44\\
  NGC$^\dagger$~\cite{wu2021ngc} & 79.16 & 91.84\\
  Sel-CL+$^\dagger$~\cite{li2022selcl} & 79.96 & 92.64\\
  \midrule
  \textbf{OURS} & \textbf{80.96} & \textbf{93.56}\\
  \bottomrule
  \end{tabular}
  \label{tab:webvision}
\end{table}

\subsection{Comparison with state-of-the-art methods}
We present the performance of LNL-flywheel, in comparison with several state-of-the-art methods under various label noise scenarios.

Tables~\ref{tab:cifar-sym-asym} presents our results on CIFAR-10 and CIFAR-100 with different types and levels of label noise.
LNL-flywheel outperforms the state-of-the-arts by a significant margin across all noise rates and all noise types.
Especially, it clearly shows robustness in difficult cases (\eg, high noise ratio, non-symmetric noise).
We ran t-SNE projection \cite{maaten2008tsne} with the responses of the networks, trained on CIFAR-10 with 90\% symmetric noise, to the validation data. Figure \ref{fig:tsne} compares the results with those of two top-performing state-of-the-arts, namely 
DivideMix \cite{li2020dividemix} and UNICON \cite{karim2022unicon}. Our method shows substantially better class separation than the other two.

Table~\ref{tab:cifar-instance} shows that our method is also competitive in instance-dependent label noise, despite not being designed against instance-dependent label noise.
InstanceGM~\cite{garg2023instancegm}, which is specially designed for instance-dependent label noise, performs slightly better than our method in CIFAR-100 with 40\% instance-dependent label noise. However, LNL-flywheel gains improvements in other various settings including Clothing1M.

Table \ref{tab:tiny} shows test accuracies on Tiny-ImageNet for symmetric noise settings.
LNL-flywheel not only outperforms the others, \eg, NCT \cite{sarfraz2021nct} and UNICON \cite{karim2022unicon}, in noisy settings but also shows comparable results in the absence of noise.
Remind that NCT and UNICON make predictions using an ensemble of multiple models.

Tables \ref{tab:clothing}, \ref{tab:webvision} demonstrate that LNL-flywheel consistently performs well on real-world noisy datasets, \ie, Clothing1M and WebVision.
In Clothing1M, SFT~\cite{wei2022sft} gains 0.04\%p over LNL-flywheel, however, we achieve significant performance improvement of SFT~\cite{wei2022sft} in CIFAR-10 and CIFAR-100, with various noise types and noise rates, \eg, 8.0\%p gain on CIFAR-100 with 40\% instance-dependent label noise.

\begin{figure}[t]
\centering
\includegraphics[width=.45\linewidth]{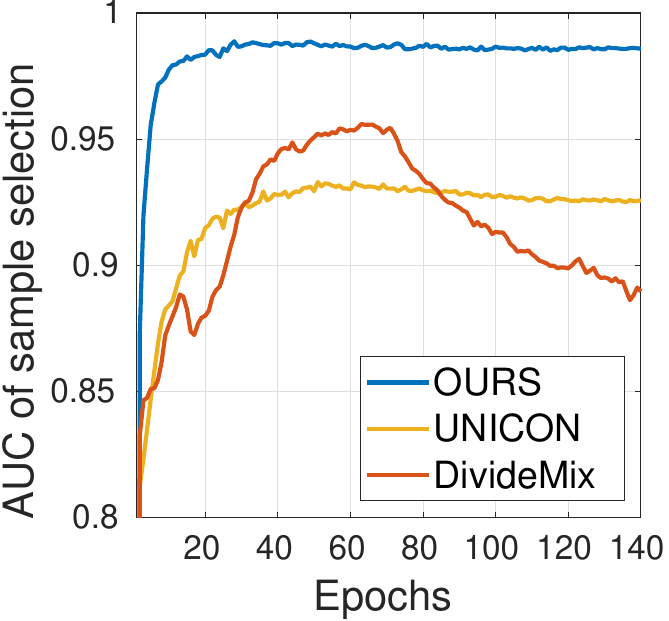}
\quad
\includegraphics[width=.45\linewidth]{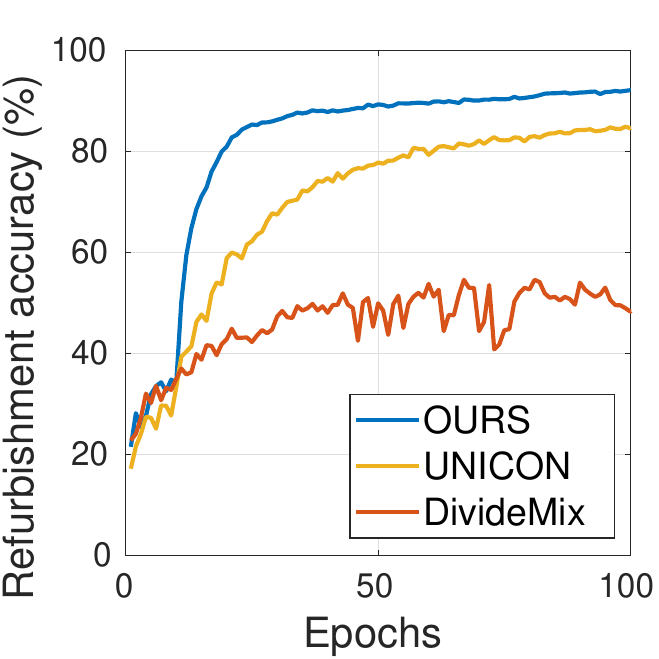}
\caption{Comparisons of clean sample selection performance (\textbf{left}) and label refurbishment accuracy (\textbf{right}) with other methods over epochs, on CIFAR-10, 90\% symmetric noise. 
}
\label{fig:functionality}
\end{figure}

\section{Analysis}
We conducted three types of analytical experiments. We will present them in this section.

\subsection{Functionality of two networks}
LNL-flywheel operates on the interconnection of two networks with different functionality. The main network acts as a distinguisher between clean and corrupted examples, while the auxiliary network as a refurbisher of the corrupted labels. The first set of analytical experiments is to investigate whether each network is functioning properly.

As an analysis associated with main network, we calculated the area under curve (AUC) representing the true and false positive rate of clean sample selection in each axis. Figure~\ref{fig:functionality}(left) shows how the AUCs evolve as training proceeds on CIFAR-10 with 90\% symmetric noise. It clearly demonstrates that our main network outperforms the others and that especially DivideMix becomes unstable after a certain number of epochs. 

In the side of auxiliary network, we probed the refurbishment performance by computing the clean label prediction accuracy over the training dataset. Figure~\ref{fig:functionality}(right) shows label refurbishment performance over epochs. As observed, our method has a significant increase in accuracy immediately after warm-up period (first 10 epochs), where two networks start to cooperate to create the virtuous cycles.

\subsection{Ablation studies}
To analyze the effect of the training components of LNL-flywheel, we conducted ablation studies and presented the results in Table \ref{tab:ablation}.

First, CR plays a crucial role for noise rejection by defining the structural constraint of the manifold in our objective function (Eq.~\ref{eq:objective}). It turns out to be indispensable for LNL-flywheel. 

The auxiliary model also proves significantly helpful in improving accuracy by allowing the entire data to be used.

In M-step of the auxiliary cycle, we estimate the label corruption probability matrix $T_c$ and calculate $\epsilon$ per training sample (Eq.~\ref{eq:epsilon}).
If we fix $\epsilon$ at $1/K$, the test accuracies degrade (see the third row), particularly exhibiting 1.8\%p drop in CIFAR-10 and 6.0\%p drop in CIFAR-100 with 40\% asymmetric noise.

In UNICON \cite{karim2022unicon}, contrastive learning (CL) was one of the key components (with CON in its name standing for contrastive learning).
Because CL is an effective feature learning technique that alleviates the memorization risk, we included CL in our framework, and it exhibits benefits on CIFAR-100 with high noise rates, although our method is not so greatly affected by CL as UNICON.

\begin{table}[t]
  \centering
  \caption{Ablation study with different training settings, symmetric (Sym.) and asymmetric (Asym.) label noise on CIFAR-10 and CIFAR-100.}
  \vspace{0.2em}
  \small
  \begin{tabular}{l|cc|c|cc|c}
    \toprule
    Dataset & \multicolumn{3}{c|}{CIFAR-10} & \multicolumn{3}{c}{CIFAR-100}\\
    \midrule
    Noise type & \multicolumn{2}{c|}{Sym.}& Asym. & \multicolumn{2}{c|}{Sym.}& Asym. \\
    Noise level & 80\% & 90\% & 40\% & 80\% & 90\% & 40\%\\
    \midrule
    w/o CR & 92.3 & 87.5 & 88.8 & 65.9 & 46.2 & 64.0\\
    w/o Aux. & 95.6 & 91.5 & 90.3 & 40.6 & 35.7 & 56.6\\
    $\epsilon$ fixed & 95.8 & 92.0 & 94.4 & \textbf{72.3} & 55.5 & 72.6\\
    w/o CL & 95.7 & \textbf{94.6} & \textbf{96.3} & 71.4 & 57.4 & \textbf{78.6}\\
    Full & \textbf{95.9} & 94.4 & 96.2 & 71.9 & \textbf{58.8} & \textbf{78.6}\\
    \bottomrule
  \end{tabular}
  \label{tab:ablation}
\end{table}

\begin{table}[t]
  \centering
  \small
  \caption{Comparison in execution time of training (per epoch) and inference (over all test data) on CIFAR-10.}
  \begin{tabular}{l|cc}
    \toprule
    Method & Training time & Inference time\\
    \midrule
    DivideMix~\cite{li2020dividemix} & 59.93\,s & 1.7423\,s\\
    UNICON~\cite{karim2022unicon} & 74.39\,s & 1.6947\,s\\
    \textbf{OURS} & \textbf{57.11\,s} & \textbf{0.9720\,s}\\
    \bottomrule
  \end{tabular}
  \label{tab:computational_cost}
\end{table}

\subsection{Computational cost}
LNL-flywheel employs only a single model, the auxiliary network, for inference, while the recent state-of-the-art methods widely use an ensemble of two DNNs. In Table \ref{tab:computational_cost}, our method shows its advantage in terms of the computational cost at testing time, spending less than 60\% of DivideMix~\cite{li2020dividemix} and UNICON~\cite{karim2022unicon}. For training, our method takes almost same amount of time as DivideMix~\cite{li2020dividemix}, while UNICON~\cite{karim2022unicon} is slower.
 
\section{Conclusion}
\label{sec:conclusion}
We presented \textit{LNL-flywheel}, a novel learning framework for robustness to noisy labels. We designed two networks which conduct two EMs cooperatively, from a considerate probabilistic reasoning.
One network distinguishes clean labels from corrupted ones, while the other refurbishes the corrupted labels. During training, the two networks collaboratively maximize a well-established objective function, making a virtuous cycle of revealing clean structure amongst data with noisy labels. By experiments, LNL-flywheel is shown to achieve state-of-the-art performance in standard benchmarks with substantial margins under various types of label noise, especially exhibiting robustness to extreme conditions such as high noise ratio and non-symmetric noise.
For inference, we use a single model, the label-refurbishing network, unlike most of recent methods that use an ensemble of multiple DNNs.

{\small
\bibliographystyle{ieee_fullname}
\bibliography{egbib}
}

\newpage

\appendix

\section{Proof of Lemma~\ref{thm:confidence-regularizer}}

For analysis in the non-parametric limit, we minimize $L$ directly \wrt \ $g_\theta(x)[\hat{y}]$'s, 
with the constraint on $g_\theta(x)[\hat{y}]$'s that they should be valid probability masses, \ie, $\sum_{\hat{y}} g_\theta(x)[\hat{y}] = 1$, $\forall x$  and $g_\theta(x)[\hat{y}] \geq 0$,  $\forall x$, $\hat{y}$. 

Relaxing  the equality constraints via Lagrange multipliers into
\[
 L_{\rm lag} = L + \sum_x \xi_x (\sum_{\hat{y}} g_\theta(x)[\hat{y}] - 1)
\]
and differentiating $L_{\rm lag}$ \wrt $g_\theta(x)[\hat{y}]$, we have 
\[
 \frac{\partial L_{\rm lag}}{\partial g_\theta(x)[\hat{y}]} = -\frac{p_{\rm data}(x,\hat{y})}{g_\theta(x)[\hat{y}]} + \lambda\frac{p_{\rm data}(x)p_{\rm data}(\hat{y})}{g_\theta(x)[\hat{y}]} +\xi_x, 
\]
which must be  zero for all $g^*_\theta(x)[\hat{y}]$'s at non-boundaries, \ie, $g^*_\theta(x)[\hat{y}] > 0$.
Therefore, 
\[
g^*_\theta(x)[\hat{y}] = \frac{1}{\xi_x}\Bigl(p_{\rm data}(x,\hat{y}) - \lambda p_{\rm data}(x)p_{\rm data}(\hat{y})\Bigr)
\]
as long as it is positive. Otherwise, $g^*_\theta(x)[\hat{y}]$ must be zero.
Overall, we can write 
\[
g^*_\theta(x)[\hat{y}] = \frac{1}{\xi_x}\Bigl(p_{\rm data}(x,\hat{y}) - \lambda p_{\rm data}(x)p_{\rm data}(\hat{y})\Bigr)_+,
\]
where $(\cdot)_+ = \max(\cdot, 0)$.
Expanding $p_{\rm data}(x,\hat{y})$ to $p_{\rm data}(x)p_{\rm data}(\hat{y}|x)$, 
we rearrange the above equation as 
\resizebox{1.0\linewidth}{!}{
  \begin{minipage}{\linewidth}
  \begin{align*}
    g^*_\theta(x)[\hat{y}] &= \frac{1}{\xi_x}\Bigl(p_{\rm data}(x)p_{\rm data}(\hat{y}|x) - \lambda p_{\rm data}(x)p_{\rm data}(\hat{y})\Bigr)_+ \\
    &= \frac{1}{Z_x}\Bigl(p_{\rm data}(\hat{y}|x) - \lambda p_{\rm data}(\hat{y})\Bigr)_+,  
  \end{align*}
  \end{minipage}
}
where $Z_x = \xi_x/p_{\rm data}(x)$.
We determine $Z_x$ so that it satisfies $\sum_{\hat{y}} g^*_\theta(x)[\hat{y}] = 1$. 

Let $p_{\rm data}(\hat{y}|x) = \gamma' p_{\pi}(\hat{y}|x) + (1-\gamma')\epsilon$, as assumed. Then, 
$g^*_\theta(x)[\hat{y}] $ becomes equal to
\[
 g^*_\theta(x)[\hat{y}] = \frac{1}{Z_x} \Bigl( \gamma' p_{\pi}(\hat{y}|x) + (1-\gamma')\epsilon - \lambda p_{\rm data}(\hat{y})
 \Bigr)_+.
\]
If $p_{\rm data}(\hat{y})$ is constant, 
by taking $\lambda = (1-\gamma') \epsilon/p_{\rm data}(\hat{y})$, 
we can cancel out the last two terms of the numerator, 
and can omit the operator $(\cdot)_+$ because $\gamma'$ and $p_{\pi}(\hat{y}|x)$ are both non\-negative. 
Consequently, $g^*_\theta(x)[\hat{y}] =  \gamma' p_{\pi}(\hat{y}|x)/Z_x$.
Finally, we find $Z_x = \gamma'$ by satisfying $\sum_{\hat{y}} g^*_\theta(x)[\hat{y}] = 1$. 

Therefore, $g^*_\theta(x)[\hat{y}]$ is reduced to $p_{\pi}(\hat{y}|x)$. 
If our network has enough capacity, $g_\theta(x)[\hat{y}]$ is believed to converge to $g^*_\theta(x)[\hat{y}]$ via an appropriate learning process \cite{goodfellow2014gan}. 

\end{document}